\documentclass{article}
\usepackage{amsmath,amssymb,amsthm}
\usepackage{bbm}
\usepackage{multirow, multicol}
\usepackage{arxiv}
\usepackage{xcolor}
\usepackage[utf8]{inputenc} % allow utf-8 input
\usepackage[T1]{fontenc}    % use 8-bit T1 fonts
\usepackage{hyperref}       % hyperlinks
\usepackage{url}            % simple URL typesetting
\usepackage{booktabs}       % professional-quality tables
\usepackage{amsfonts}       % blackboard math symbols
\usepackage{nicefrac}       % compact symbols for 1/2, etc.
\usepackage{microtype}      % microtypography
\usepackage{lipsum}
\usepackage{graphicx}
\graphicspath{ {./figures/} }
\newtheorem{theorem}{Theorem}
\newtheorem*{remark}{Remark}

\title{Conditional Conformal Risk Adaptation}
\author{
 Rui Luo \\
  Department of Systems Engineering\\
  City University of Hong Kong\\
  \texttt{ruiluo@cityu.edu.hk} \\
  \And
 Zhixin Zhou \\
  Alpha Benito Research \\
  \texttt{zzhou@alphabenito.com} \\
}
\date{}

\begin{document}

\maketitle

\begin{abstract}
Uncertainty quantification is becoming increasingly important in image segmentation, especially for high-stakes applications like medical imaging. While conformal risk control generalizes conformal prediction beyond standard miscoverage to handle various loss functions such as false negative rate, its application to segmentation often yields inadequate conditional risk control—some images experience very high false negative rates while others have negligibly small ones. We develop Conformal Risk Adaptation (CRA), which introduces a new score function for creating adaptive prediction sets that significantly improve conditional risk control for segmentation tasks. We establish a novel theoretical framework that demonstrates a fundamental connection between conformal risk control and conformal prediction through a weighted quantile approach, applicable to any score function. To address the challenge of poorly calibrated probabilities in segmentation models, we introduce a specialized probability calibration framework that enhances the reliability of pixel-wise inclusion estimates. Using these calibrated probabilities, we propose Calibrated Conformal Risk Adaptation (CCRA) and a stratified variant (CCRA-S) that partitions images based on their characteristics and applies group-specific thresholds to further enhance conditional risk control. Our experiments on polyp segmentation demonstrate that all three methods—CRA, CCRA, and CCRA-S—provide valid marginal risk control and deliver more consistent conditional risk control across diverse images compared to standard approaches, offering a principled approach to uncertainty quantification that is particularly valuable for high-stakes and personalized segmentation applications.
\end{abstract}

\section{Introduction}

Image segmentation is a fundamental computer vision task with critical applications in medical diagnostics, autonomous driving, and remote sensing. While deep learning has significantly advanced segmentation performance, reliable uncertainty quantification remains challenging but essential for safety-critical applications. Traditional evaluation metrics like Dice or IoU provide overall performance measures but fail to offer instance-wise reliability guarantees.

Conformal prediction (CP) has emerged as a powerful framework for providing distribution-free uncertainty quantification with finite-sample guarantees. It constructs prediction regions that contain the true label with a user-specified probability, regardless of the underlying data distribution. Recent work on conformal risk control (CRC) \cite{angelopoulos2024conformal} has extended this framework to handle more complex performance metrics beyond simple miscoverage, including controlling the false negative rate in segmentation tasks.

However, image segmentation presents distinct challenges when applying CRC:

\begin{enumerate}
\item \textbf{Weak conditional risk control}: While CRC guarantees marginal risk control across the dataset, the conditional risk can vary substantially between individual images. Existing CRC approaches cannot adequately address this image-specific variability.
\item \textbf{Dependence on accurate probability estimates}: CRC's effectiveness relies heavily on well-calibrated prediction probabilities. Image segmentation models often produce poorly calibrated confidence estimates, resulting in suboptimal prediction sets that fail to satisfy the intended risk constraints.
\end{enumerate}

In this paper, we first establish a fundamental connection between conformal risk control and conformal prediction through a weighted quantile approach. This connection not only provides theoretical insights but also leads to improved conditional guarantees. Building on this foundation, we make the following contributions:

\begin{enumerate}
    \item We propose a weighted quantile formulation that enables better conditional risk control across different image characteristics. This approach can be applied to any general score functions for risk control while maintaining the theoretical guarantees of conformal prediction.
    
    \item We develop Conformal Risk Adaptation (CRA), a novel segmentation score function derived from adaptive prediction sets in conformal classification. This score function allows prediction sets to adapt to the underlying confidence distribution of each image, further enhancing conditional risk control.
    
    \item To further reduce the gap between the conditional risk and the desired control level, we introduce a comprehensive calibration framework that combines probability calibration with stratified calibration strategies inspired by group conditional conformal prediction. This integrated approach improves the reliability of pixel-wise inclusion probabilities and reduces absolute error in coverage across varying image characteristics, resulting in more consistent conditional risk control.
\end{enumerate}

Our comprehensive experiments on both medical and natural image segmentation tasks demonstrate that our method provides valid risk control with significantly improved conditional guarantees compared to existing approaches.

\section{Related Work}

\subsection{Conformal Risk Control}

Conformal prediction, introduced by Vovk $et$ $al$.~\cite{vovk2005algorithmic}, provides distribution-free uncertainty quantification with finite-sample guarantees. The split conformal prediction variant \cite{lei2018distribution} has gained popularity due to its computational efficiency. This methodology has been widely applied in classification \cite{luo2024trustworthy}, regression tasks \cite{luo2024conformal}, and can be adapted to diverse real-world scenarios, such as games \cite{luo2024game}.

Recent work has extended conformal prediction to control a range of risk metrics beyond miscoverage \cite{angelopoulos2024conformal, bates2021distribution}. Angelopoulos $et$ $al$.~\cite{angelopoulos2024conformal} introduced Conformal Risk Control (CRC), which guarantees that the expected value of any monotone loss function (e.g., false negative rates in segmentation) is bounded by a user-specified tolerance. Additionally, Luo $et$ $al$. \cite{luo2024entropy} introduced entropy-based techniques to enhance the reliability of conformal methods.

To prevent the average interval length of the consistency set from inflating due to the use of identical scalars, Teneggi $et$ $al$. \cite{teneggi2023trust} proposed assigning each pixel to one of $K$ groups that share some statistical data and ensuring risk control through a convex surrogate loss. However, this approach may fail in the presence of pixel-level semantic inconsistencies, such as in CT imaging with significant anatomical variations. To address this, Teneggi $et$ $al$.  \cite{teneggi2025conformal} introduced sem-CRC, which leverages state-of-the-art segmentation models (U-Net \cite{ronneberger2015u}, nnU-Net \cite{isensee2021nnu}) and calibrates uncertainty intervals separately for each semantic group, thereby achieving more stringent clinically relevant risk control. For structured data applications, Luo $et$ $al$. \cite{luo2023anomalous} developed anomaly detection methods for graph-based scenarios, extending conformal false positive rate guarantees to network data.

In addition, Bereska $et$ $al$. \cite{bereska2025sacp} introduced Spatially-Aware Conformal Prediction (SACP), which explicitly takes into account the distance from critical structures by incorporating a locally adaptive inconsistency score, thereby achieving more conservative uncertainty estimation near key anatomical interfaces (such as tumor-vessel boundaries). Argaw $et$ $al$. \cite{argaw2022identifying} investigated heterogeneous effects in randomized clinical trials and employed joint confidence intervals derived from conformal quantile regression for prediction. He $et$ $al$. \cite{he2025training} proposed a training-aware conformal risk control method that integrates conformal risk control with conformal training, incorporating domain-specific decision thresholds and clinical risk functions into the uncertainty estimation framework. By focusing on both accuracy and uncertainty quantification during the training process, their method achieves high sensitivity and specificity while reducing the workload of the quality assurance process. Similarly, Luo $et$ $al$. \cite{luo2025volume} advanced conformal methods for multi-output regression settings.

Together, these advances move from global, fixed-group strategies to adaptive methods that align uncertainty quantification with underlying image semantics and spatial context, making conformal risk control increasingly practical for real-world clinical applications and other domains such as graph-based applications \cite{wang2025enhancing}.

\subsection{Conditional Conformal Prediction}

Although conformal methods guarantee marginal coverage, in high-stakes decision-making areas such as healthcare \cite{vazquez2022conformal}, radiation therapy \cite{he2025training}, and drug affinity \cite{gibbs2025conformal}, marginal guarantees are insufficient--significant disparities in coverage across relevant subpopulations may persist. An effective remedy is to target conditional coverage. However, without distributional assumptions, achieving conditional coverage--that is, ensuring prediction sets attain the desired coverage for every possible covariate value--is provably impossible \cite{vovk2012conditional, foygel2021limits}.

It is often necessary to relax the stringent requirement of conditional validity in favor of focusing on marginal information or multi-accuracy objectives \cite{andrews2013inference, hebert2018multicalibration, kim2019multiaccuracy, deng2023happymap, kiyani2024conformal}. Some works have explored modifying the calibration step \cite{lei2014distribution, guan2023localized, barber2023conformal}, while others have altered the initial prediction rule in order to more accurately capture the conditional distribution of $Y|X$ \cite{romano2019conformalized, sesia2021conformal, chernozhukov2021distributional, luo2024weighted}.

Some researchers have also investigated coverage under covariate shift \cite{lei2014distribution, tibshirani2019conformal, izbicki2022cd, guan2023localized, hore2023conformal, cauchois2024robust}. For instance, Gibbs $et$ $al$. \cite{gibbs2025conformal} developed a framework that achieves exact finite-sample coverage for all possible shifts. Their approach involves defining a series of tasks wherein certain classes of functions must be covered, effectively creating an interpolation between marginal and conditional coverage. This framework not only addresses the challenges posed by covariate shifts but also provides a robust solution for ensuring coverage guarantees under varying distributional changes. 

In addition, some studies have explored achieving coverage by imposing group conditional guarantees \cite{toccaceli2019combination, gupta2020distribution, ding2023class, dunn2023distribution, kiyani2024length, luo2024conformalized}. Vovk $et$ $al$. \cite{vovk2003mondrian} introduced Mondrian conformal prediction, which provides exact coverage for disjoint groups but does not accommodate overlapping subgroups. Romano $et$ $al$. \cite{romano2020malice} focused on achieving equitable coverage over disjoint protected groups. Meanwhile, Barber $et$ $al$. \cite{foygel2021limits} proposed a more flexible method to handle overlapping groups; however, this approach is computationally intensive and tends to produce overly conservative prediction intervals. More recently, Jung $et$ $al$. \cite{jung2023batch} introduced a quantile regression method based on a linear combination of subgroup indicator functions to enhance conditional coverage for overlapping groups. Nonetheless, this method relies on distributional assumptions and may still struggle to offer adequate coverage in finite-sample settings.

Additional work has aimed to improve conditional coverage by learning features from the data. For instance, Yuksekgonul $et$ $al$. \cite{yuksekgonul2023beyond} proposed a density-based atypicality concept to enhance calibration and conditional coverage with respect to input atypicality. Kiyani $et$ $al$. \cite{kiyani2024conformal} investigated learning partitions of the covariate space, such that points within the same partition are similar in terms of their prediction sets, with the goal of improving conditional validity. Kaur $et$ $al$. \cite{kaur2025conformal} examined general patterns of miscoverage in standard conformal prediction and developed a two-dimensional statistic that consistently demonstrates its effectiveness in settings that extend beyond well-defined groups.
Moreover, Aabesh $et$ $al$. \cite{bhattacharyya2024group} combined Weighted Conformal Prediction with Group Weighting to ensure predictive coverage amid covariate shift, while Prinzhorn $et$ $al$. \cite{prinzhorn2024conformal} introduced a conformal method integrating time series decomposition with component-specific exchangeability for reliable uncertainty estimates in forecasting. For neural network applications on graph data, recent works \cite{luo2025conformal, tang2025enhanced} have developed specialized techniques that enhance reliability and coverage guarantees.

Our method builds upon and improves these approaches by offering tighter finite-sample risk control guarantees without distributional assumptions, while allowing for overlapping groups and more flexible function classes.

\section{Preliminaries}

\subsection{Problem Setup}

We consider the image segmentation task where each input image $X_i$ is associated with a ground truth segmentation mask $Y_i \subset \{1, 2, \ldots, N_i\}$, where $N_i$ represents the total number of pixels in image $X_i$. Each pixel is indexed by $j \in \{1, 2, \ldots, N_i\}$. 
Our goal is to construct a prediction set $\hat{C}(X_i) \subset \{1, 2, \ldots, N_i\}$ that controls the false negative rate in expectation, ensuring: 
\begin{equation}
\mathbb{E}\left[1-\frac{\left|\hat{C}(X_i) \cap Y_i \right|}{\left|Y_i\right|}\right] \leq \alpha,
\label{eq:coverage_goal}
\end{equation}
where $\alpha \in (0,1)$ is a user-specified risk level. The expectation in \eqref{eq:coverage_goal} reflects the average performance of the prediction set $\hat{C}(X_i)$ over random draws of the test data. This metric represents the proportion of true positive pixels that are incorrectly excluded in the prediction set, which is particularly important in applications like medical imaging where missing regions of interest can have serious consequences.

\subsection{Conformal Risk Control}\label{subsec:crc_intro}

Conformal risk control (CRC) \cite{angelopoulos2024conformal} extends conformal prediction to control the expected value of any monotone loss function. For segmentation, CRC controls false negative rate by  applying a threshold to the pixel-wise probabilities produced by a base model. The optimal threshold value is determined through a calibration procedure on a held-out dataset.

Given a prediction model $\hat{p}(X_i) = (\hat{p}_1(X_i), \ldots, \hat{p}_{N_i}(X_i))$, where $\hat{p}_j(X_i)$ estimates $\mathbb{P}(j \in Y_i | X_i)$ for pixel $j$, the CRC approach:

\begin{enumerate}
    \item Defines prediction sets $\hat{C}(X_i, \tau) = \{j : \hat{p}_j(X_i) \geq \tau\}$ for a threshold $\tau$.
    
    \item Computes the calibrated threshold:
    \begin{equation}
    \tau' = \inf \left\{\tau: \frac{1}{n+1} \sum_{i\in\mathcal{I}_\text{cal}} \left(1 - \frac{|\hat{C}(X_i,\tau) \cap Y_i|}{|Y_i|}\right) + \frac{B}{n+1} \leq \alpha\right\}
    \end{equation}
    where $n$ is the size of the calibration set $\mathcal{I}_\text{cal}$, $B$ is the upper bound of the loss function, and $\alpha$ is the desired risk level. 
    
    \item Returns the final prediction set $\hat{C}(X_i, \tau')$.
\end{enumerate}

This approach guarantees that $\mathbb{E}[1 - |\hat{C}(X_i, \tau') \cap Y_i|/|Y_i|] \leq \alpha$ over the data distribution, providing a distribution-free control of the false negative rate, as established by Theorem 1 in Angelopoulos $et$ $al$. \cite{angelopoulos2024conformal}.

\section{Methodology}

\subsection{Conformal Risk Adaptation (CRA)}
\label{subsec:cra}

We introduce Conformal Risk Adaptation (CRA), which shares merits with adaptive prediction sets (APS) \cite{romano2020classification} for conformal classification that adapts the prediction set to the underlying probability distribution of the segmentation model while providing statistical guarantees for risk control.

For each image $X_i$, we define an adaptive prediction set:
\begin{equation}
\label{eq:adaptive_set}
\hat{C}(X_i, \alpha') = \arg\min_{C \subset \{1,2,...,N_i\}} \{|C| : \sum_{j \in C} \hat{p}_j(X_i) \geq (1-\alpha') \sum_{j=1}^{N_i} \hat{p}_j(X_i)\},
\end{equation}
which selects the smallest set of pixels that captures at least $(1-\alpha')$ of the total predicted probability mass. Similar to APS, our approach accounts for the relative ranking of pixel probabilities rather than applying uniform thresholds across all images, allowing for more adaptive and powerful prediction regions.

The intuition behind this approach relates to the expected coverage, where $\mathbb{E}[|C\cap Y_i|]=\sum_{j\in C}p_j(X_i)$ and $\mathbb{E}[|Y_i|]=\sum_{j\in[N_i]}p_j(X_i)$. By controlling the ratio of these expected values, we can effectively control the false negative rate. By controlling the ratio of these expected values, we can effectively control the false negative rate.

For efficient computation, we define a conformity score that captures each pixel's position in this adaptive ranking:
\begin{equation}
\label{eq:cra_score}
s(X_i, j) = \frac{\sum_{j'=1}^{N_i} \hat{p}_{j'}(X_i) \mathbbm{1}\{\hat{p}_{j'}(X_i) \leq \hat{p}_j(X_i)\}}{\sum_{j'=1}^{N_i} \hat{p}_{j'}(X_i)}.
\end{equation}
This score increases with the pixel's predicted probability, such that pixels with higher scores are more likely to be included in the prediction set.

\begin{theorem}
\label{thm:adaptive_set_equivalence}
The adaptive prediction set $\hat{C}(X_i, \alpha')$ is equivalent to the set of pixels with conformity scores above a threshold: $\hat{C}(X_i, \alpha') = \{j : s(X_i, j) \geq 1-\alpha'\}$.
\end{theorem}

\begin{proof}
Let's define $S(X_i, 1-\alpha') = \{j : s(X_i, j) \geq 1-\alpha'\}$. We need to show that $\hat{C}(X_i, \alpha') = S(X_i, 1-\alpha')$.

The conformity score $s(X_i, j)$ in Equation~\eqref{eq:cra_score} represents the normalized cumulative probability mass of all pixels with probability less than or equal to $\hat{p}_j(X_i)$. A score of $s(X_i, j) \geq 1-\alpha'$ means that pixel $j$ is among the highest probability pixels whose cumulative probability mass accounts for at least $(1-\alpha')$ of the total probability mass.

Let's examine what the condition $s(X_i, j) \geq 1-\alpha'$ means in terms of probabilities:
\begin{equation*}
\frac{\sum_{j'=1}^{N_i} \hat{p}_{j'}(X_i) \mathbbm{1}\{\hat{p}_{j'}(X_i) \leq \hat{p}_j(X_i)\}}{\sum_{j'=1}^{N_i} \hat{p}_{j'}(X_i)} \geq 1-\alpha'.
\end{equation*}

If we arrange all pixels in order of increasing probability and include pixel $j$ and all pixels with higher probabilities in our set $C$, then:
\begin{equation*}
\frac{\sum_{j' \in C} \hat{p}_{j'}(X_i)}{\sum_{j'=1}^{N_i} \hat{p}_{j'}(X_i)} \geq 1-\alpha',
\end{equation*}
which is equivalent to:
\begin{equation*}
\sum_{j' \in C} \hat{p}_{j'}(X_i) \geq (1-\alpha')\sum_{j'=1}^{N_i} \hat{p}_{j'}(X_i).
\end{equation*}

This is precisely the definition of $\hat{C}(X_i, \alpha')$ in Equation~\eqref{eq:adaptive_set}. By construction, the set $S(X_i, 1-\alpha')$ contains exactly those pixels that satisfy the above inequality and is the smallest such set (as pixels are added in descending order of probability). Therefore, $\hat{C}(X_i, \alpha') = S(X_i, 1-\alpha')$.
\end{proof}

\subsection{Weighted Quantile Formulation for Risk Control}
\label{subsec:weighted_quantile}

For segmentation tasks, we can formulate risk control as a weighted quantile problem, providing both theoretical insights and computational efficiency benefits.

For each pixel $j$ in image $X_i$, we define a conformity score that measures how likely a pixel is to be included in the prediction set:
\begin{equation}
\label{eq:conformity_score}
s(X_i, j) = \hat{p}_j(X_i),
\end{equation}
where $\hat{p}_j(X_i)$ estimates the probability that pixel $j$ belongs to the target class.

With this definition, the calibrated threshold $\tau'$ for controlling the false negative rate can be computed as:
\begin{equation}
\label{eq:calibrated_threshold}
\tau' = \inf \left\{\tau: \frac{n}{n+1}\sum_{i \in \mathcal{I}_\text{cal}} \sum_{j \in Y_i} \frac{\mathbbm{1}\{s(X_i, j) < \tau\}}{|Y_i|} + \frac{B}{n+1} \leq \alpha\right\},
\end{equation}
where $n$ is the size of the calibration set, $B$ is the upper bound of the loss function, and $\alpha$ is the desired risk level.

This formulation can be interpreted as finding the $[(n+1)\alpha-B]/n$-th quantile of the weighted distribution of scores $\{(s(X_i,j), 1/|Y_i|): i\in \mathcal{I}_\text{cal}, j\in Y_i\}$, where $1/|Y_i|$ serves as the weight for each score $s(X_i,j)$. This weighted quantile formulation applies to any valid score function, including the proposed CRA as defined in Equation~\eqref{eq:conformity_score}.

\begin{theorem}
\label{thm:weighted_quantile}
For any score function $s(X_i, j)$ that ranks pixels based on their likelihood of inclusion, the optimal threshold $\tau'$ that controls the expected false negative rate at level $\alpha$ can be computed as the solution to the weighted quantile formulation in Equation~\eqref{eq:calibrated_threshold}.
\end{theorem}

\begin{proof}
Let's define the false negative rate for a new test point $(X_{n+1}, Y_{n+1})$ as:
\begin{equation*}
\text{FNR}(X_{n+1}, Y_{n+1}, \tau) = 1 - \frac{|\hat{C}_\tau(X_{n+1}) \cap Y_{n+1}|}{|Y_{n+1}|},
\end{equation*}
where $\hat{C}_\tau(X_{n+1}) = \{j : s(X_{n+1}, j) \geq \tau\}$ is the prediction set with threshold $\tau$.

Our goal is to find a threshold $\tau'$ such that $\mathbb{E}[\text{FNR}(X_{n+1}, Y_{n+1}, \tau')] \leq \alpha$.

Following the conformal risk control framework \cite{angelopoulos2024conformal}, if we have calibration data $(X_i, Y_i)_{i=1}^n$ and use the risk estimator:
\begin{equation*}
\hat{R}(\tau) = \frac{1}{n}\sum_{i=1}^n \text{FNR}(X_i, Y_i, \tau) = \frac{1}{n}\sum_{i=1}^n \left(1 - \frac{|\hat{C}_\tau(X_i) \cap Y_i|}{|Y_i|}\right),
\end{equation*}
then choosing $\tau'$ as:
\begin{equation*}
\tau' = \inf\left\{\tau : \frac{n}{n+1}\hat{R}(\tau) + \frac{B}{n+1} \leq \alpha\right\},
\end{equation*}
guarantees that $\mathbb{E}[\text{FNR}(X_{n+1}, Y_{n+1}, \tau')] \leq \alpha$.

Expanding the definition of $\hat{R}(\tau)$:
\begin{align*}
\tau' &= \inf\left\{\tau : \frac{n}{n+1}\frac{1}{n}\sum_{i=1}^n \left(1 - \frac{|\hat{C}_\tau(X_i) \cap Y_i|}{|Y_i|}\right) + \frac{B}{n+1} \leq \alpha\right\} \\
&= \inf\left\{\tau : \frac{1}{n+1}\sum_{i=1}^n \left(1 - \frac{|\hat{C}_\tau(X_i) \cap Y_i|}{|Y_i|}\right) + \frac{B}{n+1} \leq \alpha\right\}.
\end{align*}

For our score function $s(X_i, j) = \hat{p}_j(X_i)$, we have:
\begin{align*}
1 - \frac{|\hat{C}_\tau(X_i) \cap Y_i|}{|Y_i|} &= 1 - \frac{|\{j \in Y_i : s(X_i, j) \geq \tau\}|}{|Y_i|} \\
&= \frac{|\{j \in Y_i : s(X_i, j) < \tau\}|}{|Y_i|} \\
&= \sum_{j \in Y_i} \frac{\mathbbm{1}\{s(X_i, j) < \tau\}}{|Y_i|}.
\end{align*}

Substituting this back, we get:
\begin{equation*}
\tau' = \inf\left\{\tau : \frac{1}{n+1}\sum_{i=1}^n \sum_{j \in Y_i} \frac{\mathbbm{1}\{s(X_i, j) < \tau\}}{|Y_i|} + \frac{B}{n+1} \leq \alpha\right\},
\end{equation*}
which is exactly Equation~\eqref{eq:calibrated_threshold}.
\end{proof}

\begin{remark}
Traditional implementations of risk control methods use grid search to find the threshold \cite{angelopoulos2024conformal}, which is computationally inefficient and depends on the grid resolution. The weighted quantile formulation enables efficient algorithms for threshold computation, resulting in more accurate thresholds without the limitations of grid-based search.
\end{remark}

\subsection{Probability Calibration for Accurate Probability Estimation}

While CRC only relies on the relative ordering of pixel probabilities (meaning monotone calibration would not affect its performance), our CRA approach depends on accurate estimation of the total probability mass $\sum_{j=1}^{N_i} \hat{p}_{j}(X_i)$. Therefore, we introduce a probability calibration framework specifically designed for segmentation models.

Given predicted probabilities $\hat{p}_j(X_i)$ from a base segmentation model, we seek a calibration function $f: [0,1] \rightarrow [0,1]$ that satisfies:

\begin{enumerate}
    \item \textbf{Monotonicity}: For any $\hat{p}_a \leq \hat{p}_b$, we require $f(\hat{p}_a) \leq f(\hat{p}_b)$.

    \item \textbf{Probability Matching}: The calibrated probabilities should minimize the empirical cross-entropy loss:
    \begin{equation}
    \mathcal{L}_{\text{emp}}(f) = -\frac{1}{|\mathcal{I}_{\text{val}}|}\sum_{i \in \mathcal{I}_{\text{val}}}\sum_{j=1}^{N_i} \left[y_{ij}\log f(\hat{p}_j(X_i)) + (1-y_{ij})\log(1-f(\hat{p}_j(X_i)))\right],
    \end{equation}
    where $y_{ij} \in \{0,1\}$ indicates whether pixel $j$ in image $X_i$ belongs to the target class.
\end{enumerate}

This calibration function $f$ is applied to all predicted probabilities $\hat{p}_j(X_i)$ across all images and pixels simultaneously. We train this function on a separate validation set $\mathcal{I}_{\text{val}}$, distinct from both the training data used for the base model and the calibration set $\mathcal{I}_{\text{cal}}$ that will later be used for conformal calibration. This separation ensures that the probability calibration step does not interfere with the subsequent conformal calibration process.

\begin{figure}[t]
    \centering
    \begin{minipage}[t]{0.55\textwidth}
        \centering
        \includegraphics[width=\linewidth]{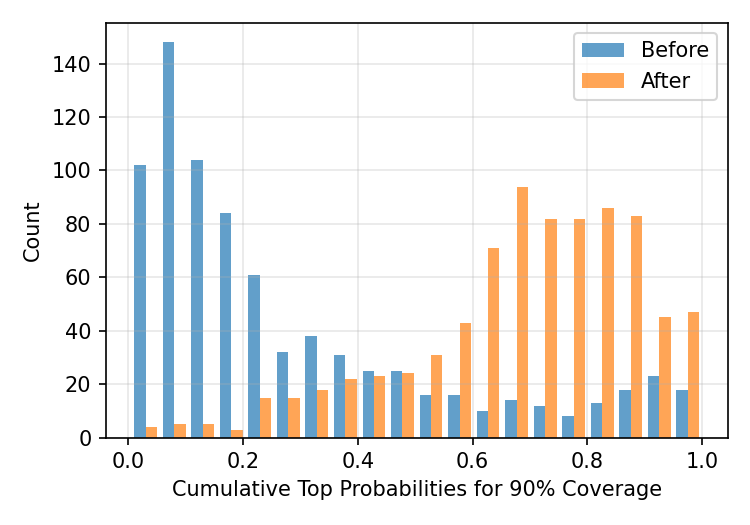}
        \label{fig:prob_calibration}
    \end{minipage}
    \hfill
    \begin{minipage}[t]{0.4\textwidth}
        \centering
        \includegraphics[width=\linewidth]{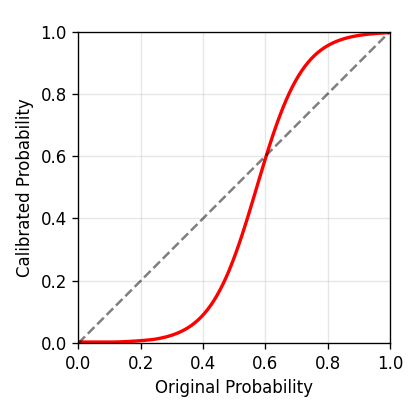}
        \label{fig:calibration_curve}
    \end{minipage}
    \caption{Left: Distribution of probability mass proportion required to achieve $(1-\alpha)$ coverage before and after calibration. For each sample $X_i$, we compute the minimum proportion of total probability mass ($\sum \hat{p}_j(X_i)$) needed to include at least $(1-\alpha)$ of the true positive pixels $Y_i$. Right: Calibration curve showing predicted vs. calibrated probabilities. The red curve shows our isotonic regression calibration function, correcting for overconfidence in mid-range probabilities.}
    \label{fig:combined_calibration}
\end{figure}

\subsection{Stratified Calibration for Improved Conditional Risk Control}

To further reduce the gap between the conditional risk and the desired control level, we propose a stratified calibration approach inspired by group conditional conformal prediction.

We define a partitioning function $g: \mathcal{X} \rightarrow \{1,2,...,K\}$ that assigns each image to one of $K$ strata based on the sum of predicted probabilities across all pixels in the image, $\sum_{j=1}^{N_i} \hat{p}_{j}(X_i)$. This sum serves as a proxy for the model's overall confidence in its segmentation of the image.

For each stratum $k \in \{1,2,...,K\}$, we compute a stratum-specific calibrated threshold:
\begin{equation}
\alpha'_k = \inf \left\{\alpha: \frac{n_k}{n_k+1}\sum_{i \in \mathcal{I}_\text{cal}: g(X_i)=k} \left(1 - \frac{|\hat{C}(X_i, \alpha) \cap Y_i|}{|Y_i|}\right) + \frac{B}{n_k+1} \leq \alpha\right\},
\end{equation}
where $n_k = |\{i \in \mathcal{I}_\text{cal}: g(X_i)=k\}|$ is the number of calibration samples in stratum $k$.

For a new test image $X_{n+1}$, we determine its stratum $k = g(X_{n+1})$ by calculating its total probability mass and apply the corresponding threshold $\alpha'_k$.

The stratified calibration approach provides valid risk control within each stratum:
\begin{equation}
\mathbb{E}[1 - |\hat{C}(X_i, \alpha'_k) \cap Y_i|/|Y_i| | g(X_i)=k] \leq \alpha.
\end{equation}
This approach improves conditional risk control by adapting the calibration thresholds to images with different levels of predicted probability mass, which is particularly important in segmentation tasks where the model's confidence can vary significantly across different images. Images with similar total predicted probabilities tend to exhibit similar risk characteristics, allowing for more tailored risk control.

\section{Experiments}

In our experiment, we evaluate our proposed conditional conformal risk control approaches, including CRA, CCRA and CCRA-S, against the baseline CRC on polyp segmentation in colonoscopy images. We use the same datasets (ETIS, CVC-ClinicDB, CVC-ColonDB, EndoScene, and Kvasir) and model setup (PraNet) as previous studies \cite{angelopoulos2024conformal, mossina2025conformal}. The model was trained on 1000 images, with the remaining 798 images used for evaluation. We use 70\% for calibration and 30\% for testing, repeating this split over 100 trials to thoroughly assess whether our conditional approaches provide improved uncertainty quantification compared to standard CRC.

\begin{table}[ht]
    \centering
    \caption{Marginal Coverage and Coverage Gap Results for Different Significance Levels and Methods}
    \label{tab:coverage_results}
    \begin{tabular}{cccc}
        \toprule
        $\alpha$ & Method & Marginal Coverage & Coverage Gap \\
        \midrule
        \multirow{4}{*}{0.05} & CRC & 0.953 (0.149) & 0.078 (0.127) \\
        & CRA & 0.953 (0.141) & 0.076 (0.119) \\
        & CCRA & 0.954 (0.111) & 0.066 (0.090) \\
        & CCRA-S & 0.963 (0.106) & \textbf{0.062 (0.087)} \\
        \midrule
        \multirow{4}{*}{0.10} & CRC & 0.900 (0.251) & 0.157 (0.196) \\
        & CRA & 0.901 (0.211) & 0.143 (0.155) \\
        & CCRA & 0.901 (0.161) & 0.113 (0.115) \\
        & CCRA-S & 0.908 (0.158) & \textbf{0.101 (0.122)} \\
        \midrule
        \multirow{4}{*}{0.20} & CRC & 0.801 (0.346) & 0.264 (0.224) \\
        & CRA & 0.802 (0.294) & 0.243 (0.165) \\
        & CCRA & 0.801 (0.208) & 0.158 (0.135) \\
        & CCRA-S & 0.809 (0.200) & \textbf{0.138 (0.146)} \\
        \bottomrule
    \end{tabular}
\end{table}

\begin{figure}[tb]
    \centering
    \begin{minipage}[t]{0.4\linewidth}
        \centering
        \includegraphics[width=\linewidth]{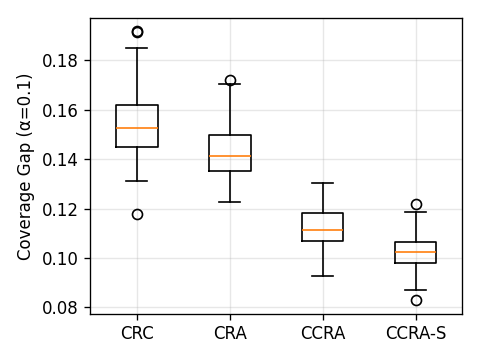}
    \end{minipage}%
    \begin{minipage}[t]{0.6\linewidth}
        \centering
        \includegraphics[width=\linewidth]{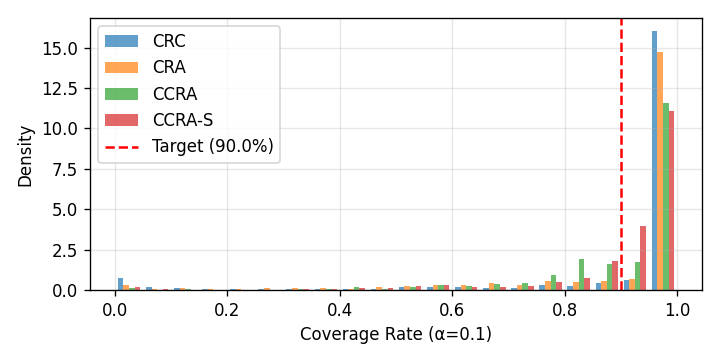}
    \end{minipage}
    \caption{Left: Coverage gap distribution for all methods at significance level $\alpha = 0.1$ (target coverage 90\%). The boxplots show the distribution of absolute differences between achieved coverage and target coverage across 100 experimental trials. Lower values indicate better conformity to the target coverage. Our proposed CCRA and CCRA-S methods demonstrate smaller coverage gaps and less variance compared to the baseline CRC, indicating more reliable uncertainty quantification. Right: Coverage distribution for all methods at significance level $\alpha = 0.1$. The histogram shows the density of coverage values achieved across all test samples in all experimental trials. The vertical red dashed line indicates the target coverage of 90\%. Our proposed methods, especially CCRA-S, produce distributions more tightly centered around the target coverage, demonstrating better calibration compared to the baseline CRC method, which shows a wider, more dispersed distribution.}
    \label{fig:coverage_combined}
\end{figure}

For each test sample, we calculate the actual coverage as the proportion of true labels captured by the prediction region. The coverage gap is then defined as the $l_1$ distance between the achieved coverage and the target coverage $(1-\alpha)$. A smaller coverage gap indicates better conditional risk control.

Our experimental results demonstrate that the proposed conditional approaches outperform the standard CRC method in terms of the coverage gap. Figure~\ref{fig:coverage_combined} shows that CCRA and CCRA-S achieve smaller coverage gaps with less variance, indicating more reliable uncertainty quantification. The coverage distributions in the same figure further confirm this finding, with our methods producing distributions more tightly centered around the target coverage. 

When evaluating performance across different $\alpha$ values (Figure~\ref{fig:coverage_gap_vs_alpha}), our conditional approaches consistently maintain smaller coverage gaps throughout the entire range, with CCRA-S demonstrating the most robust performance. This suggests that our methods provide more reliable uncertainty quantification regardless of the desired confidence level.

The qualitative comparison in Figure~\ref{fig:visual_comparison} demonstrates CCRA-S's ability to generate more informative prediction regions. We fix $\alpha=0.1$ (target coverage 90\%) and randomly select test images where CRC's coverage spans a wide range (from 85\% to 99.5\%). For each selected image, we compare the prediction sets generated by both methods. While CRC's coverage fluctuates substantially across different polyp images, CCRA-S consistently produces prediction regions closer to the target 90\% coverage level on the same images. This stable performance across diverse clinical cases makes CCRA-S particularly valuable for medical applications where consistent and well-calibrated uncertainty quantification is crucial for reliable decision support.

Notably, the stratification approach in CCRA-S allows for adaptive thresholds based on predicted polyp size, addressing a key limitation of standard conformal prediction methods that apply a single threshold across heterogeneous data. This adaptation enables CCRA-S to provide more reliable uncertainty estimates for both small and large polyps, potentially improving clinical decision support.

\begin{remark}
    It is important to note that the coverage gap cannot be reduced to zero in practical settings for two key reasons. First, the estimated probabilities $\hat{p}$ are inherently inaccurate, introducing a baseline error that propagates through the prediction pipeline. Second, the ratio $|\hat{C}(X_i)\cap Y_i|/|Y_i|$ is a random variable in our setting; its $l_1$ distance from the target coverage $(1-\alpha)$ has a positive expected value even under ideal conditions. These fundamental limitations establish a theoretical lower bound on the achievable coverage gap, highlighting the significance of the improvements demonstrated by our proposed methods.
\end{remark}

\begin{figure}[tb]
    \centering
    \includegraphics[width=0.75\textwidth]{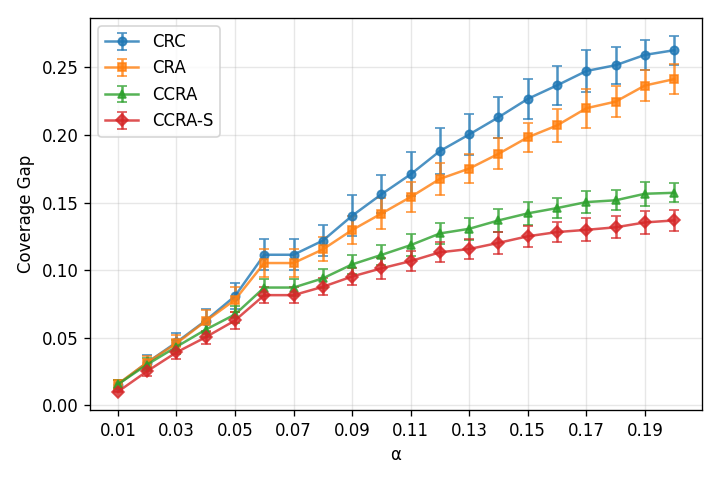}
    \caption{Coverage gap versus significance level $\alpha$ for all methods. The plot shows how the mean coverage gap (with standard deviation error bars) varies with different values of $\alpha$ from 0.01 to 0.20. Our proposed methods consistently achieve smaller coverage gaps than the baseline CRC across the entire range of significance levels, with CCRA-S demonstrating the optimal performance. This indicates that our conditional approaches provide more reliable uncertainty quantification regardless of the desired confidence level.}
    \label{fig:coverage_gap_vs_alpha}
\end{figure}

\begin{figure}[tb]
    \centering
    \includegraphics[width=\textwidth]{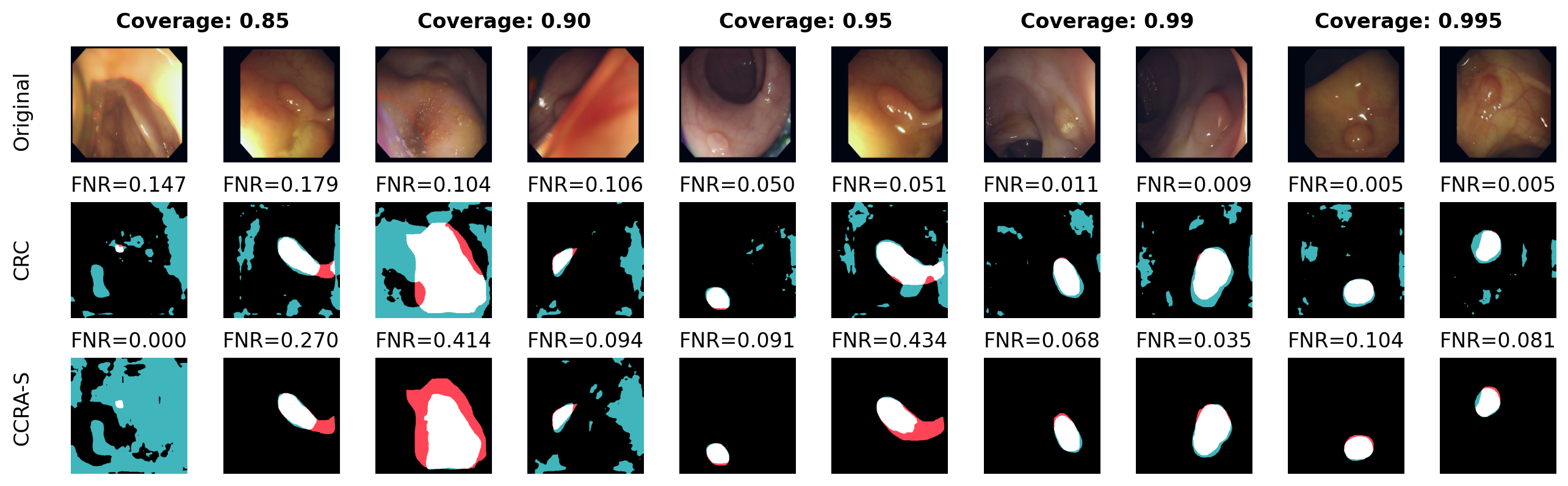}
    \caption{Qualitative comparison of CRC and CCRA-S prediction sets at significance level $\alpha = 0.1$. Each column pair shows examples at specific CRC coverage levels (from left to right: 85\%, 90\%, 95\%, 99\%, and 99.5\%), while CCRA-S coverage remains approximately 90\% across all samples. The top row displays original polyp images, the middle row shows CRC prediction sets, and the bottom row presents CCRA-S prediction sets. White pixels indicate true positives, red pixels show false negatives, and teal pixels represent false positives. FNR values (False Negative Rate = 1 - coverage) are shown for each prediction. Note how CCRA-S maintains more consistent coverage and false negatives across different images, while CRC's performance varies substantially depending on image characteristics, demonstrating the advantage of our conditional risk control approaches.}
    \label{fig:visual_comparison}
\end{figure}

\section{Conclusion}

In this paper, we introduced Conformal Risk Adaptation (CRA), a novel approach for improving conditional risk control in image segmentation tasks. CRA addresses a critical limitation in standard conformal methods where some images experience much higher false negative rates than others. We further enhanced CRA through a specialized probability calibration framework, resulting in Calibrated Conformal Risk Adaptation (CCRA) and a stratified variant (CCRA-S) that partitions images based on their total predicted probability, enabling group-specific thresholds that deliver more consistent risk control across diverse images.

The theoretical foundation of our work establishes a fundamental connection between conformal risk control and conformal prediction through a weighted quantile approach. This connection is applicable to any conformal risk control score function and provides the mathematical framework that guarantees valid risk bounds. 

Our methodological contributions bridge the gap between theoretical conformal methods and practical computer vision applications, offering a principled approach to uncertainty quantification in segmentation. The resulting framework provides distribution-free guarantees on false negative rates with improved conditional risk control, making it particularly valuable for medical imaging applications where both overall performance and per-patient reliability are essential for clinical decision-making.

\end{document}